\documentclass[10pt]{article}

\usepackage[utf8]{inputenc} 
\usepackage[T1]{fontenc}    
\usepackage{hyperref}       
\usepackage{url}            
\usepackage{booktabs}       
\usepackage{amsfonts}       
\usepackage{nicefrac}       
\usepackage{microtype}      
\usepackage{float}
\usepackage{enumerate}
\usepackage{caption}
\usepackage{fullpage}
\usepackage{tgpagella}

\usepackage{mathrsfs}%
\usepackage{amsmath}
\usepackage{amssymb}
\usepackage{natbib}
\usepackage{graphicx}
\PassOptionsToPackage{x11names}{xcolor}
\usepackage{xcolor}
\usepackage{jmlrutils}
\usepackage{nameref}
\usepackage{authblk}
\newcommand{\E}{{\mathbb E}}
\DeclareMathOperator*{\argmin}{arg\,min}
\DeclareMathOperator{\R}{\mathbb{R}}

\newcommand{\cv}{{\mbox{CV}_{loo}}}
\newcommand{\cvr}{{\mbox{CV}_{ro}}}

\newcommand{\x}{\mathbf{x}}
\newcommand{\w}{\mathbf{w}}
\newcommand{\y}{\mathbf{y}}
\newcommand{\avec}{\mathbf{a}}
\newcommand{\bvec}{\mathbf{b}}
\newcommand{\dvec}{\mathbf{d}}
\newcommand{\evec}{\mathbf{e}}
\newcommand{\fvec}{\mathbf{f}}

\newcommand{\kvec}{\mathbf{k}}
\newcommand{\hvec}{\mathbf{h}}
\newcommand{\uvec}{\mathbf{u}}
\newcommand{\vvec}{\mathbf{v}}

\providecommand{\nor}[1]{\left\lVert {#1} \right\rVert}
\providecommand{\scal}[2]{\langle{#1},{#2}\rangle}
\newcommand{\hh}{{\cal H}}
\newcommand{\Xn}{{\mathbf X }}
\newcommand{\yn}{{\mathbf y }}
\newcommand{\Kn}{{\mathbf K }}
\newcommand{\Si}{S_i}

\newcommand{\PP}{{\mathbb P}}

\title{For interpolating kernel machines, minimizing the norm of the ERM solution minimizes stability}

\author[1,2]{Akshay Rangamani \thanks{\url{arangam@mit.edu}}}
\author[1]{Lorenzo Rosasco \thanks{\url{lrosasco@mit.edu}}}
\author[1,2]{Tomaso Poggio \thanks{\url{tp@csail.mit.edu}}}

\affil[1]{Center for Brains, Minds, and Machines, MIT}
\affil[2]{McGovern Institute for Brain Research, MIT}

\date{}

\begin{document}


\maketitle

\begin{abstract}


\noindent  We study the average $\cv$ stability of kernel ridge-less regression
  and derive corresponding risk bounds.  We show that the
  interpolating solution with minimum norm minimizes a bound on $\cv$
  stability, which in turn is controlled by the condition number of
  the empirical kernel matrix.  The latter can be characterized in the
  asymptotic regime where both the dimension and cardinality of the
  data go to infinity.  Under the assumption of random kernel matrices,
  the corresponding test error should be expected to follow a double descent curve.

\end{abstract}
\vskip0.1in

\section{Introduction} \label{sec:intro}

Statistical learning theory studies the learning properties of machine
learning algorithms, and more fundamentally, the conditions under which
learning from finite data is possible. In this context, classical learning
theory focuses on the size of the hypothesis space in terms of
different complexity measures, such as combinatorial dimensions,
covering numbers and Rademacher/Gaussian complexities \citep{shai,bou}.
Another more recent approach is based on defining suitable notions of
stability with respect to perturbation of the data \citep{BE:2001,kn2002}.
In this view, the continuity of the process that maps data to estimators is crucial,
rather than the complexity of the hypothesis space.  Different notions
of stability can be considered, depending on the data perturbation and
metric considered \citep{kn2002}. Interestingly, the stability and complexity approaches to
characterizing the {\em learnability} of problems are not at odds with
each other, and can be shown to be equivalent(
\citet{PogRifMukNiy04} and \citet{10.5555/1756006.1953019}).

In modern machine learning overparameterized models, with a larger number of parameters
than the size of the training data, have become common. The ability of these models to
generalize is well explained by classical statistical learning theory as long as some form of
regularization is used in the training process \citep{buhvan21, steinwart2008support}. However, it
was recently shown - first for deep networks \citep{zhang}, and more recently for
kernel methods \citep{Belkin15849} - that learning is possible in the absence of regularization,
i.e., when perfectly fitting/interpolating the data.
Much recent work in statistical learning theory has tried to find theoretical ground for this
empirical finding. Since learning using models that interpolate is not exclusive to
deep neural networks, we study generalization in the presence of
interpolation in the case of  kernel methods. We study both linear and
kernel least squares problems in this paper.

\paragraph*{Our Contributions:}

\begin{itemize}
	\item We characterize the generalization properties of
          interpolating solutions for
		  linear and kernel least squares problems using a stability approach. While
		  the (uniform) stability properties of regularized kernel methods are well known \citep{BE:2001},
		  we study interpolating solutions of the unregularized  ("ridgeless")
		  regression problems.
	\item We obtain an upper bound on the stability of interpolating solutions, and show
	      that this upper bound is minimized by the minimum norm interpolating solution.
	      This also means that among all interpolating solutions, the minimum norm solution has
	      the best test error. In particular, the same conclusion is also true for gradient descent,
	      since it converges to the minimal norm solution in the setting we consider, see e.g. \cite{rosvil15}.
	\item Our stability bounds show that the average stability of the minimum norm solution
	      is controlled by the condition number of the empirical kernel matrix. It is well known that
	      the numerical stability of the least squares solution is governed by the condition number of
	      the associated kernel matrix (see \citet{DescentCondition}). Our results show that the condition number also controls
	      stability (and hence, test error) in a statistical sense.
\end{itemize}

\paragraph*{Organization:}
In section \ref{sec:sl_erm_definitions}, we
introduce basic ideas in statistical learning and empirical risk minimization, as well as
the notation used in the rest of the paper. In section \ref{sec:stability_err_bds}, we briefly
recall some definitions of stability. In section \ref{sec:CVkernels}, we study the stability of
interpolating solutions to kernel least squares and show that the minimum norm solutions
minimize an upper bound on the stability. In section \ref{sec:discussion} we discuss
our results in the context of recent work on high dimensional
regression. We conclude
in section \ref{sec:conclusions}.

\section{Statistical Learning and  Empirical Risk Minimization} \label{sec:sl_erm_definitions}
We begin by recalling the basic ideas in statistical learning theory.  In this setting, $X$ is the space of
features, $Y$ is the space of targets or labels, and there is an unknown probability distribution $\mu$
on the product space $Z =  X \times Y$.  In the following, we consider $X=  \R^d$ and $Y = \R$.
The distribution $\mu$ is fixed but unknown, and we are given a
training set $S$ consisting of $n$ samples (thus $|S| = n$) drawn
i.i.d.  from the probability distribution on $Z^n$,
$ S = (z_i)_{i=1}^n=(\x_i,y_i)_{i=1}^n. $ Intuitively, the goal of
supervised learning is to use the training set $S$ to ``learn'' a
function $f_S$ that evaluated at a new value $\x_{new}$ should predict
the associated value of $y_{new}$, i.e.
$ y_{new} \approx f_S(\x_{new}).$

 The loss is a function $V : {\cal F} \times Z \to [0,\infty)$, where ${\cal F}$ is the space of
measurable functions from $X$ to $Y$, that measures how well a function performs on a data point. We define  a hypothesis space
${\cal H}\subseteq {\cal F}$ where algorithms  search for solutions.
With the above notation, the {\it expected risk} of $f$  is defined as $I[f] = \E_z V(f,z)$
which is the expected loss on a new sample drawn according to the data distribution $\mu$.
In this setting, statistical learning can be seen as the problem of finding an approximate
minimizer of the expected risk
given a training set $S$.  A classical approach to derive an approximate solution
is empirical risk minimization (ERM) where we minimize the empirical risk $I_S[f] = \frac{1}{n} \sum_{i=1}^n V(f,z_i)$.

A natural error measure for our ERM solution $f_S$ is the {expected excess risk}
$
\E_S[I[f_S]- \min_{f\in \hh} I[f]].
$
Another common error measure is the {expected generalization error/gap} given by
$
\E_S[I[f_S]-  I_S[f_S]].
$
These two error measures are closely related since, the expected excess risk is easily bounded by the expected generalization error (see Lemma~\ref{expriskgen}).

\subsection{Kernel Least Squares  and Minimal Norm Solution} \label{subsec:KLS_min_norm}
The focus in this paper is on the kernel least squares problem.
We assume the loss function $V$ is the square loss, that is,
$V(f,z)= (y- f(\x))^2.$ The hypothesis space is assumed to be a
reproducing kernel Hilbert space, defined by a positive definite kernel $K: X\times X\to \R$ or an associated feature map $\Phi:X\to {\cal H}$, such
that $K(\x,\x')= \scal{\Phi(\x)}{\Phi(\x')}_{\cal H}$ for all $\x,\x'\in X$,
where $\scal{\cdot}{\cdot}_{\cal H}$ is the inner product in ${\cal H}$. In
this setting, functions are linearly parameterized, that is there
exists $w\in {\cal H}$ such that $f(\x)=\scal{w}{\Phi(\x)}_{\cal H}$ for all
$x\in X$.


The ERM problem typically has multiple solutions, one of which is the minimal norm solution:
\begin{equation}\label{NOLS_ps}
f_S^\dagger=\argmin_{f\in {\cal M}} \nor{f}_\hh, \quad \quad {\cal M}= \argmin_{f\in \hh} \frac{1}{n} \sum_{i=1}^n (f(\x_i) - y_i)^2.
\end{equation}
Here $\nor{\cdot}_\hh$ is the norm on $\hh$ induced by the inner product. The minimal norm solution can be shown to be unique and
satisfy a representer theorem, that is  for all $\x\in X$:
\begin{equation}\label{NOLS_rep}
f_S^\dagger(\x)=\sum_{i=1}^n K(\x,\x_i) c_{S, i}, \quad \quad \mathbf{c}_S= \Kn^{\dagger} \yn
\end{equation}
where $\mathbf{c}_S= (c_{S, 1}, \dots, c_{S, n}), \y = (y_1 \ldots y_n) \in \R^n$, $\Kn$ is the $n$ by $n$ matrix with entries $\Kn_{ij}=K(\x_i,\x_j)$, $i,j=1, \dots, n$,
and $\Kn^\dagger$ is the Moore-Penrose pseudoinverse of $\Kn$.
If we assume $n \le d$ and that we have $n$ linearly independent data features, that is the rank of $\Xn$ is $n$, then  it is possible to show  that
for many kernels one can replace $\Kn^{\dagger} $ by $\Kn^{-1}$ (see Remark~\ref{inv}). Note that invertibility is necessary and sufficient for interpolation.
That is, if $\Kn$ is invertible, $f_S^\dagger(\x_i) =  y_i$ for all $i=1, \dots, n$, in which case the training error in~\eqref{NOLS_ps} is zero.
\begin{remark}[Pseudoinverse for underdetermined linear systems]
A simple yet relevant example are linear functions $f(\x)=\w^\top \x $, that correspond to ${\cal H}=\R^d$ and $\Phi$ the identity map.
If the rank of $\Xn \in \R^{d \times n}$ is $n$, then any interpolating solution $\w_S$ satisfies $\w_S^\top \x_i  = y_i$ for all $i=1, \dots, n$,  and the minimal norm solution, also called Moore-Penrose solution,  is given by $ (\w_S^\dagger)^\top= \yn^\top \Xn^\dagger $
where the pseudoinverse $ \Xn^\dagger$ takes the form $\Xn^\dagger= \Xn^\top(\Xn \Xn^\top)^{-1}.$
\end{remark}

\begin{remark}[Invertibility of translation invariant kernels] \label{remark:invertible_kernel}\label{inv}
Translation invariant kernels are a family of kernel functions given by $K(\x_1,\x_2) = k(\x_1 - \x_2)$
where $k$ is an even function on $\R^d$. Translation invariant kernels are Mercer kernels
(positive semidefinite) if the Fourier transform of $k(\cdot)$ is non-negative.
For {Radial Basis Function} kernels ($K(\x_1,\x_2) = k(||\x_1 - \x_2||)$)
we have the additional property due to  Theorem 2.3 of
\cite{Micchelli86} that for distinct points $\x_1, \x_2, \ldots, \x_n \in \R^d$ the
kernel matrix $\Kn$ is non-singular and thus invertible.
\end{remark}

 The above discussion is directly related to regularization approaches.

\begin{remark}[Stability and Tikhonov regularization]\label{Tik}
Tikhonov regularization is used to prevent potential unstable behaviors.
In the above setting, it corresponds to replacing Problem~\eqref{NOLS_ps} by
$
\min_{f\in \hh}\frac 1 n \sum_{i=1}^n (f(\x_i)-y_i)^2+\lambda \nor{f}^2_\hh
$
where the corresponding unique solution is given by
$
f_S^\lambda(\x)=\sum_{i=1}^n K(\x,\x_i) c_i, \quad \quad \mathbf{c}= (\Kn+\lambda \mathbf{I}_n )^{-1} \yn.
$
In contrast to ERM solutions,  the above approach prevents interpolation.
The properties of the corresponding estimator are well known. In this paper, we complement these results focusing on the case $\lambda\to 0$.
\end{remark}

Finally, we end by recalling the connection between minimal norm and the gradient descent.

\begin{remark}[Minimum norm and gradient
  descent] \label{remark:gd_min_norm} In our setting, it is well
  know the both batch and stochastic gradient iterations converge
  exactly to the minimal norm solution, when multiple solutions exist,
  see e.g. \cite{rosvil15}. Thus, a study of  the  properties of
  minimal norm solutions explain the properties of the solution to
  which gradient descent converges. In particular, when ERM has multiple
  interpolating solutions,  gradient descent
  converges to a solution that minimizes a bound on stability, as we show next.
\end{remark}

\section{Error Bounds via Stability} \label{sec:stability_err_bds}
In this section, we recall basic results relating the learning and stability properties of Empirical Risk Minimization (ERM).
Throughout the paper, we assume that ERM achieves a minimum, albeit
the extension to almost minimizer is possible \citep{Mukherjee2006} and important for
exponential-type loss functions \citep{Foundations}. We do not assume the expected risk to achieve a minimum.
Since we will be considering leave-one-out stability in this section, we look at solutions to ERM 
over the complete training set $S=\{z_1, z_2, \dots, z_n\}$ and the leave one out
training set $\Si=\{z_1, z_2, \dots, z_{i-1}, z_{i+1}, \dots, z_n \}$

The excess risk of ERM can  be easily related to its stability properties.
Here,  we follow the definition laid out in \cite{Mukherjee2006} and say that an
algorithm is Cross-Validation leave-one-out ($\cv$) stable in expectation,
if there exists $\beta_{CV}>0$ such that  for all $i=1, \dots, n$,
\begin{equation}\label{cvst}
 \E_S[V(f_{\Si}, z_i)- V(f_S, z_i)] \le \beta_{CV}.
\end{equation}
This definition is justified by the following result that bounds the excess risk of a
learning algorithm by its average $\cv$ stability.

%

\begin{lemma}[Excess Risk  \& $\cv$ Stability]\label{expriskgen} \label{lemmone}
For all $i=1, \dots, n$,
\begin{equation}\label{cvlooexcess}
\E_S[I[f_{\Si}]-\inf_{f\in \hh} I[f]]\le
\E_S[V(f_{\Si}, z_i)- V(f_S, z_i)].
\end{equation}
\end{lemma}



\begin{remark}[Connection to  uniform stability and other notions of stability]
  Uniform stability, introduced by \cite{BE:2001}, corresponds in our
  notation to the assumption that there exists $\beta_u>0$ such that
  for all $i=1, \dots, n$,
  $ \sup_{z\in Z} |V(f_{\Si}, z)- V(f_S,z)|\le \beta_u.  $ Clearly
  this is a strong notion implying most other definitions of
  stability.  We note that there are number of different notions of
  stability. We refer the interested reader to \cite{kn2002} ,
  \cite{Mukherjee2006}.
\end{remark}

We present the proof of Lemma \ref{expriskgen} in Appendix A.2 due to lack of space.
In Appendix A, we also discuss other definitions of stability and their connections to
concepts in statistical learning theory like generalization and learnability.

\section{$\cv$ Stability of Kernel Least Squares}
\label{sec:CVkernels}

In this section we analyze the expected $\cv$  stability of
interpolating solutions to the kernel least squares problem, and obtain an upper bound
on their stability. We show that this upper bound on the expected $\cv$ stability is
smallest for the minimal norm interpolating solution \eqref{NOLS_ps} when compared to
other interpolating solutions to the kernel least squares problem.

We have a dataset $S=\{(\x_i, y_i)\}_{i=1}^n$ and we want to find a
mapping $f \in \hh$, that minimizes the empirical least squares
risk. Here $\hh$ is a reproducing kernel hilbert space (RKHS) defined
by a positive definite kernel $K: X\times X\to \R$. All interpolating
solutions are of the form
$\hat{f}_S (\cdot) = \sum_{j=1}^n \hat{c}_{S,j} K(\x_j, \cdot)$, where
$\hat{\mathbf{c}}_S = \Kn^\dagger \y +(\mathbf{I} - \Kn^\dagger
\Kn)\vvec$. Similarly, all interpolating solutions on the leave one
out dataset $\Si$ can be written as
$\hat{f}_{\Si} (\cdot) = \sum_{j=1, j \neq i}^n \hat{c}_{\Si,j}
K(\x_j, \cdot)$, where
$\hat{\mathbf{c}}_{\Si} = \Kn_{\Si}^\dagger \y_i +(\mathbf{I} -
\Kn_{\Si}^\dagger \Kn_{\Si})\vvec_i$.  Here $\Kn, \Kn_{\Si}$ are the
empirical kernel matrices on the original and leave one out datasets
respectively.  We note that when $\vvec = \mathbf{0}$ and
$\vvec_i = \mathbf{0}$, we obtain the minimum norm interpolating
solutions on the datasets $S$ and $\Si$.

\begin{theorem}[Main Theorem]\label{thm:main}
Consider the kernel least squares problem with a bounded kernel
and bounded outputs $y$, that is there exist $\kappa, M >0$ such that
\begin{equation}\label{boundedness}
K(\x,\x')\le \kappa,\quad \quad \quad  |y|\le M,
\end{equation}
almost surely.   Then for any interpolating solutions $\hat{f}_{\Si}, \hat{f}_S$,
\begin{equation}\label{eqn:kernel_cv_theorem}
\E_S[V(\hat{f}_{\Si}, z_i) - V(\hat{f}_S, z_i)] \leq\beta_{CV} ( \Kn^\dagger, \y, \vvec, \vvec_i )
\end{equation}
This bound $\beta_{CV}$ is minimized when $\vvec=\vvec_i= \mathbf{0}$, which corresponds to
the minimum norm interpolating solutions $f_S^\dagger, f_{\Si}^\dagger$.
For the minimum norm solutions we have $\beta_{CV} = C_1 \beta_1 + C_2 \beta_2$,
where $\beta_1 = \E_S \left[ ||\Kn^{\frac{1}{2}}||_{op}
||\Kn^\dagger||_{op} \times \textrm{cond}(\Kn) \times ||\y||\right]$
and, $\beta_2 = \E_S \left[ ||\Kn^{\frac{1}{2}}||_{op}^2
||\Kn^\dagger||_{op}^2 \times (\textrm{cond}(\Kn))^2 \times ||\y||^2 \right]$,
and $C_1,C_2$ are absolute constants that do not depend on either $d$ or $n$.
\end{theorem}

In the above theorem $||\Kn||_{op}$ refers to the operator norm of the kernel matrix $\Kn$,
$||\y||$ refers to the standard $\ell_2$ norm for $\y \in \R^n$, and $\textrm{cond}(\Kn)$ is the
condition number of the matrix $\Kn$.

We can combine the above result with Lemma \ref{lemmone} to obtain the following bound
on excess risk for minimum norm interpolating solutions to the kernel least squares problem:

\begin{corollary}
The excess risk of the minimum norm interpolating kernel least squares solution can
be bounded as:
\[ \E_S \left[ I[f_{\Si}^\dagger] - \inf_{f \in \hh} I[f] \right] \leq C_1 \beta_1 + C_2 \beta_2 \]
where $\beta_1, \beta_2$ are as defined previously.
\end{corollary}

\begin{remark}[Underdetermined Linear Regression]\label{rem:linreg}
In the case of underdetermined linear regression, ie, linear regression where the dimensionality is
larger than the number of samples in the training set, we can prove a version of Theorem \ref{thm:main}
with $\beta_1=\E_S \left[ \nor{\Xn^\dagger}_{op} \nor{\y} \right]$
and $\beta_2= \E_S \left[ \nor{\Xn^\dagger}_{op}^2 \nor{\y}^2 \right]$. Due to space constraints, we present
the proof of the results in the linear regression case in Appendix \ref{app:linreg_stability}.
\end{remark}

\subsection{Key lemmas}

In order to prove Theorem \ref{thm:main} we make use of the following lemmas
to bound the $\cv$ stability using the norms and the difference of the solutions.

\begin{lemma}\label{ls_lip}
Under assumption~\eqref{boundedness}, for all $i=1. \dots, n$, it holds that
$$
 \E_{S}[V(\hat{f}_{\Si}, z_i) - V(\hat{f}_S , z_i)] \le \E_{S}\left[ \left( 2M + \kappa \left( \nor{\hat{f}_S}_\hh + \nor{\hat{f}_{\Si}}_\hh \right) \right) \times \kappa \nor{\hat{f}_S-\hat{f}_{\Si}}_\hh \right]
\label{lorenzo2}
$$
\end{lemma}

\begin{proof}
We begin, recalling that the square loss is locally Lipschitz, that is for all $y, a,a'\in \R$, with
$$
|(y-a)^2- (y-a')^2| \le (2|y| + |a|+|a'|))|a-a'|.
$$
If we apply this result to $f,f'$ in a RKHS $\hh$,
$$
|(y-f(\x))^2- (y-f'(\x))^2| \le \kappa(2M + \kappa \left(\nor{f}_\hh+\nor{f'}_\hh \right) ) \nor{f-f'}_\hh.
$$
using the basic properties of a RKHS that for all $f\in\hh$
\begin{equation}\label{H2infty}
|f(\x)| \le \nor{f}_\infty
\le \kappa \nor{f}_\hh
\end{equation}
In particular, we can plug $\hat{f}_{\Si}$ and $\hat{f}_S$ into the above inequality,
and the almost positivity of ERM \citep{Mukherjee2006} will allow us to drop the absolute
value on the left hand side. Finally the desired result follows by taking the expectation over $S$.
\end{proof}

Now that we have bounded the $\cv$ stability using the norms and the difference of the solutions,
we can find a bound on the difference between the solutions to the kernel least squares problem.
This is our main stability estimate.

\begin{lemma} \label{lemma:kernel_stability}
Let $\hat{f}_S, \hat{f}_{\Si}$ be any interpolating kernel least squares solutions on
the full and leave one out datasets (as defined at the top of this section),
then $\nor{\hat{f}_S - \hat{f}_{\Si} }_{\hh} \leq B_{CV}(\Kn^\dagger, \y, \vvec, \vvec_i)$, and
$B_{CV}$ is minimized when $\vvec=\vvec_i =\mathbf{0}$, which corresponds to the minimum
norm interpolating solutions $f_S^\dagger, f_{\Si}^\dagger$.

Also, 
\begin{equation}\label{eqn:perturbation_norm}
\nor{f_S^\dagger - f_{\Si}^\dagger }_{\hh} \leq \nor{\Kn^\frac{1}{2}}_{op} \nor{\Kn^\dagger}_{op}
\times \textrm{cond} (\Kn) \times \nor{\y}
\end{equation}
\end{lemma}

Since the minimum norm interpolating solutions minimize both
$\nor{\hat{f}_S}_\hh + \nor{\hat{f}_{\Si}}_\hh$ and
$\nor{\hat{f}_S - \hat{f}_{\Si}}_\hh$ (from lemmas \ref{ls_lip},
\ref{lemma:kernel_stability}), we can put them together to prove
theorem \ref{thm:main}. In the following section we provide the proof
of Lemma \ref{lemma:kernel_stability}.

\paragraph{Simulation:} In order to illustrate that the minimum norm
interpolating solution is the best performing interpolating solution
we ran a simple experiment on a linear regression problem. We
synthetically generated data from a linear model $\y=\w^\top \Xn$,
where $\Xn \in \R^{d \times n}$ was i.i.d $\mathcal{N}(0,1)$. The
dimension of the data was $d=1000$ and there were $n=200$ samples in
the training dataset. A held out dataset of $50$ samples was used to
compute the test mean squared error (MSE). Interpolating solutions were
computed as
$\hat{\w}^\top=\y^\top \Xn^\dagger + \vvec^\top (\mathbf{I} - \Xn
\Xn^\dagger)$ and the norm of $\vvec$ was varied to obtain the
plot. The results are shown in Figure \ref{mse_vs_norm}, where we can
see that the training loss is 0 for all interpolants, but test MSE
increases as $||\vvec||$ increases, with
$(\w^\dagger)^\top = \y^\top \Xn^\dagger$ having the best
performance. The figure reports results averaged over $100$ trials.

\begin{figure}
	\centering
	\includegraphics[width=0.6\linewidth]{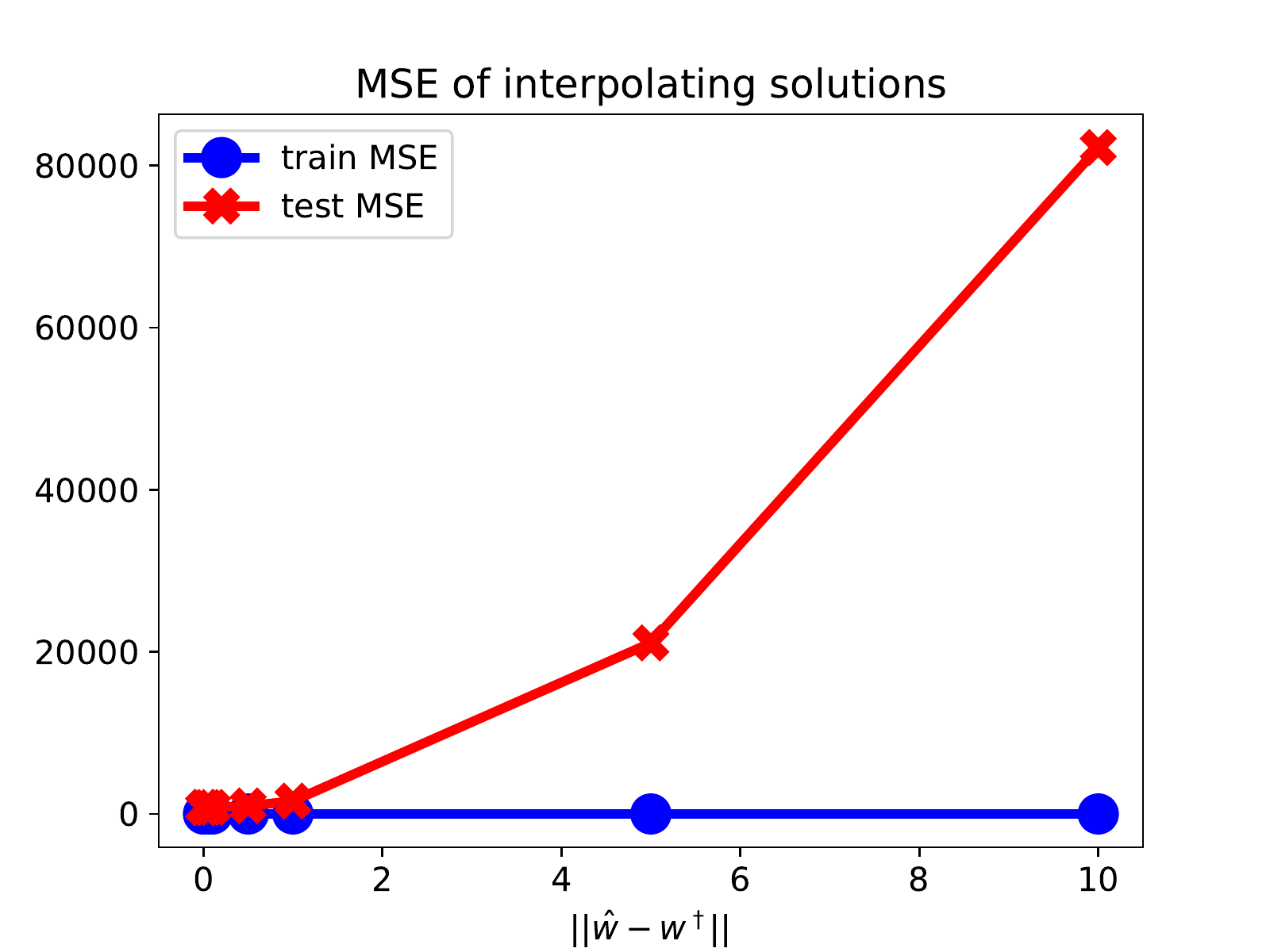}
	\caption{Plot of train and test mean squared error (MSE) vs distance between an interpolating
	solution $\hat{\w}$ and the minimum norm interpolant $\w^\dagger$ of a linear regression problem.
	Data was synthetically generated as $\y=\w^\top \Xn$, where $\Xn \in \R^{d \times n}$ with i.i.d $\mathcal{N}(0,1)$ entries and
	 $d=1000, n=200$. Other interpolating solutions were computed
	 as $\hat{\w}=\y^\top \Xn^\dagger + \vvec^\top (\mathbf{I} - \Xn \Xn^\dagger)$ and
	 the norm of $\vvec$ was varied to obtain the plot. Train MSE is 0 for all interpolants, but
	 test MSE increases as $||\vvec||$ increases, with $\w^\dagger$ having the best performance. This plot
	 represents results averaged over $100$ trials.}
	\label{mse_vs_norm}
\end{figure}



\subsection{Proof of Lemma \ref{lemma:kernel_stability}}

We can write any interpolating solution to the kernel regression problem as $\hat{f}_S(\x) = \sum_{i=1}^n \hat{c}_{S,i} K(\x_i,\x)$
where $\hat{\mathbf{c}}_S = \Kn^\dagger \y + (\mathbf{I} - \Kn^\dagger \Kn) \vvec$, and  $\Kn\in \R^{n \times n}$ is the kernel matrix  $K$ on $S$ and $\vvec$ is any vector in $\R^n$.
i.e. $\Kn_{ij} = K(\x_i,\x_j)$, and $\y \in \R^n$ is the vector $\y = [y_1 \ldots y_n]^\top$.

Similarly, the coefficient vector for the corresponding interpolating solution to the problem
over the leave one out dataset $\Si$ is $\hat{\mathbf{c}}_{\Si} = (\Kn_{\Si})^\dagger \y_i + (\mathbf{I} - (\Kn_{\Si})^\dagger \Kn_{\Si}) \vvec_i$. Where $\y_i = [y_1, \ldots, 0, \ldots y_n]^\top$ and $\Kn_{\Si}$ is the kernel matrix $\Kn$
with the $i^{\textrm{th}}$ row and column set to zero, which is the
kernel matrix for the leave one out training set.

We define $\avec = [-K(\x_1,\x_i), \ldots, -K(\x_n,\x_i)]^\top \in \R^n$
and $\bvec \in \R^n$ as a one-hot column vector
with all zeros apart from the $i^{\textrm{th}}$ component
which is $1$.
Let $\avec_* = \avec+K(\x_i,\x_i)\bvec$. Then, we have:
\begin{equation}
\begin{split}
\Kn_* &= \Kn + \avec\bvec^\top \\
\Kn_{\Si} &= \Kn_* + \bvec\avec_*^\top
\end{split}
\end{equation}
That is, we can write $\Kn_{\Si}$ as a rank-2 update to $\Kn$.
This can be verified by simple algebra, and using the fact that $K$
is a symmetric kernel. Now we are interested in bounding $||\hat{f}_S - \hat{f}_{\Si}||_\hh$.
For a function $h(\cdot)=\sum_{i=1}^m p_i K(\x_i,\cdot) \in \hh$
we have $||h||_\hh = \sqrt{\mathbf{p}^\top \Kn \mathbf{p}}
= ||\Kn^{\frac{1}{2}} \mathbf{p}||$. So we have:
\begin{equation} \label{eqn:fn_norm_diff}
\begin{split}
    ||\hat{f}_S - \hat{f}_{\Si}||_\hh &= || \Kn^{\frac{1}{2}} (\hat{\mathbf{c}}_S - \hat{\mathbf{c}}_{\Si})|| \\
    &= ||\Kn^{\frac{1}{2}} (\Kn^\dagger \y + (\mathbf{I} - \Kn^\dagger \Kn) \vvec - (\Kn_{\Si})^\dagger \y_i - (\mathbf{I} - (\Kn_{\Si})^\dagger \Kn_{\Si}) \vvec_i)|| \\
    &= ||\Kn^{\frac{1}{2}} (\Kn^\dagger \y - (\Kn_{\Si})^\dagger \y + y_i (\Kn_{\Si})^\dagger \bvec \\
    &+ (\mathbf{I} - \Kn^\dagger \Kn) (\vvec - \vvec_i) + (\Kn^\dagger \Kn - (\Kn_{\Si})^\dagger \Kn_{\Si})\vvec_i)|| \\
    &= ||\Kn^{\frac{1}{2}} [ (\Kn^\dagger - (\Kn_{\Si})^\dagger) \y + (\mathbf{I} - \Kn^\dagger \Kn) (\vvec - \vvec_i) + (\Kn^\dagger \Kn - (\Kn_{\Si})^\dagger \Kn_{\Si})\vvec_i]||
\end{split}
\end{equation}
Here we make use of the fact that $(\Kn_{\Si})^\dagger \bvec = \mathbf{0}$. If $\Kn$ has full rank (as in Remark \ref{remark:invertible_kernel}), we see that $\bvec$
lies in the column space of $\Kn$ and $\avec$ lies in the column
space of $\Kn^\top$. Furthermore, $\beta =
1 +\bvec^\top \Kn^\dagger \avec =
= 0$. Using Theorem 6 in \citep{Meyer1973} (equivalent to formula 2.1 of
\citet{baksalary2003revisitation}) with $\kvec = \Kn^\dagger \avec, \hvec = \bvec^\top \Kn^\dagger$, we obtain:
\begin{equation} \label{eqn:proof_step1}
\begin{split}
    (\Kn_*)^\dagger &= \Kn^\dagger - \kvec\kvec^\dagger \Kn^\dagger - \Kn^\dagger \hvec^\dagger \hvec + (\kvec^\dagger \Kn^\dagger \hvec^\dagger ) \kvec \hvec \\
    &= \Kn^\dagger - \kvec \kvec^\dagger \Kn^\dagger - \Kn^\dagger \hvec^\dagger \hvec - \kvec \hvec \\
    &= \Kn^\dagger - \Kn^\dagger \hvec^\dagger \hvec
\end{split}
\end{equation}

Above, we use the fact that the operator norm of a rank 1 matrix is given by $||\uvec \vvec^\top||_{op} = ||\uvec||\times ||\vvec||$ and that $\kvec = -\bvec$.
Also, using the corresponding formula from List 2 of \citet{baksalary2003revisitation}, we have $\Kn_*^\dagger \Kn_* = \Kn^\dagger \Kn - \kvec \kvec^\dagger$.

Next, we see that for $\Kn_{\Si}$, $\avec_*$ lies in the column space of $\Kn_*^\top$,
and $\bvec$ does not lie in the column space of $\Kn_*$. This means
we can use Theorem 5 in \citep{Meyer1973} (equivalent to formula 2.4
in \citep{baksalary2003revisitation}) to obtain the expression for $(\Kn_{\Si})^\dagger$

\begin{equation}\label{eqn:proof_step2}
    (\Kn_{\Si})^\dagger = \Kn_*^\dagger - \nu^{-1}( \phi \Kn_*^\dagger \evec \evec^\top + \eta \dvec \fvec^\top) + \nu^{-1}( \lambda \Kn_*^\dagger \evec \fvec^\top  - \lambda \dvec \evec^\top)
\end{equation}

Also, using the corresponding formula from List 2 of \cite{baksalary2003revisitation}, we
have $(\Kn_{\Si})^\dagger \Kn_{\Si} = \Kn_*^\dagger \Kn_*$, which implies that
$\Kn^\dagger \Kn - (\Kn_{\Si})^\dagger \Kn_{\Si} = \kvec \kvec^\dagger$.

Now let us define all the terms in equation ~\eqref{eqn:proof_step2}. We omit some algebraic simplification to save some space.
\begin{equation}
    \begin{split}
        \dvec &= \Kn_*^\dagger \bvec = \Kn^\dagger ( \mathbf{I} - \hvec^\dagger \hvec ) \bvec \\
        \evec &= (\Kn_*^\dagger)^\top \avec_* = (\hvec^\dagger \hvec - \mathbf{I}) \bvec \\
        \fvec &= (\mathbf{I} - \Kn_* \Kn_*^\dagger) \bvec = \hvec^\dagger \hvec \bvec \\
        \lambda &= 1+ \avec_*^\top \Kn_*^\dagger \bvec = 1+ \evec^\top \bvec =\bvec^\top \hvec^\dagger \hvec \bvec \\
        \phi &= \fvec^\top \fvec = \bvec^\top \hvec^\dagger \hvec \bvec = \lambda \\
        \eta &= \evec^\top \evec = 1 - \lambda \\
        \nu &= \lambda^2 + \eta \phi = \lambda
    \end{split}
\end{equation}
This means we can simplify equation ~\eqref{eqn:proof_step2} as
\begin{equation}
\begin{split}
        \Kn_*^\dagger - (\Kn_{\Si})^\dagger &=  \nu^{-1}( \phi \Kn_*^\dagger \evec \evec^\top + \eta \dvec \fvec^\top) - \nu^{-1}( \lambda \Kn_*^\dagger \evec \fvec^\top  - \lambda \dvec \evec^\top) \\
        &= \lambda^{-1} \left( \lambda \Kn_*^\dagger \evec (\evec - \fvec)^\top + \eta \dvec \fvec^\top + \lambda \dvec \evec^\top  \right) \\
        &= \Kn_*^\dagger \evec (\evec - \fvec)^\top + \dvec (\evec - \fvec)^\top + \lambda^{-1} \dvec \fvec^\top
\end{split}
\end{equation}

Putting together ~\eqref{eqn:proof_step1}, ~\eqref{eqn:proof_step2} (and after some algebraic simplification) we get:
\begin{equation}
\begin{split}
\Kn^\dagger - (\Kn_{\Si})^\dagger &= \Kn^\dagger - \Kn_*\dagger + \Kn_*\dagger - (\Kn_{\Si})^\dagger \\
&= \Kn^\dagger \hvec^\dagger \hvec + \lambda^{-1} \Kn^\dagger ( \mathbf{I} - \hvec^\dagger \hvec) \bvec \bvec^\top \hvec^\dagger \hvec \\
&= (\Kn^\dagger_{ii})^{-1} \hvec^\top \hvec
\end{split}
\end{equation}


Plugging in these calculations into equation \ref{eqn:fn_norm_diff} we get:
\begin{equation}
\begin{split}
    ||\hat{f}_S - \hat{f}_{\Si}||_\hh &= ||\Kn^{\frac{1}{2}} [ (\Kn^\dagger - (\Kn_{\Si})^\dagger) \y + (\mathbf{I} - \Kn^\dagger \Kn) (\vvec - \vvec_i) + (\Kn^\dagger \Kn - (\Kn_{\Si})^\dagger \Kn_{\Si})\vvec_i]|| \\
    &\leq ||\Kn^{\frac{1}{2}}||_{op} \left( ||(\Kn^\dagger - (\Kn_{\Si})^\dagger) \y||  + ||(\mathbf{I} - \Kn^\dagger \Kn) (\vvec - \vvec_i) || + ||\kvec \kvec^\dagger \vvec_i||  \right) \\
    &\leq ||\Kn^{\frac{1}{2}}||_{op}(B_0 + 2 ||\vvec - \vvec_i|| + ||\vvec_i||)
\end{split}
\end{equation}
We see that the right hand side is minimized when $\vvec = \vvec_i = \mathbf{0}$.
We have also computed $B_0 = ||\Kn^\dagger||_{op} \times cond(\Kn) \times ||\y||$, which concludes the proof of Lemma \ref{lemma:kernel_stability}.

\section{Remark and Related Work}\label{sec:discussion}

In the previous section we obtained bounds on the $\cv$ stability of interpolating
solutions to the kernel least squares problem. Our kernel least squares results can be compared
with stability bounds for regularized ERM (see
Remark~\ref{Tik}). Regularized ERM has a strong stability guarantee in
terms of a uniform stability bound  which turns out to be inversely proportional to the
regularization parameter $\lambda$ and the sample size $n$
\citep{BE:2001}. However, this estimate becomes vacuous as
$\lambda\to 0$.  In this paper, we establish a bound on average stability,
and show that this bound is minimized when the minimum norm ERM solution is
chosen. We study average stability since one can expect
worst case scenarios where the minimum norm is arbitrarily
large (when $n \approx d$). Our main finding is in fact the relationship between minimizing the norm of the
ERM solution and minimizing a bound on stability.

This leads to a second observation, namely, that we can consider the limit
of our risk bounds as both the sample size ($n$) and the dimensionality
of the data ($d$) go to infinity.  This is a classical setting in
statistics which allows us to use results from random matrix theory
\citep{MarchenkoPastur}.  In particular, for linear kernels the
behavior of the smallest eigenvalue of the kernel matrix (which appears in our bounds) can be
characterized in this asymptotic limit. Here the dimension
of the data coincides with the number of parameters in the
model. Interestingly, analogous results hold for more
general kernels \citep{2010arXiv1001.0492E} where the asymptotics are taken
with respect to the number  and dimensionality  of the data.
These results predict a double descent curve for the condition number as found in practice,
see Figure~\ref{CondNumberrbf}.

Recently, there has been a surge of interest in studying linear and kernel least squares
models, since classical results focus on situations where constraints
or penalties that prevent interpolation are added to the empirical risk.  For example,
  high dimensional linear regression is considered in
\cite{2019arXiv190805355M, 2019arXiv190308560H, bartlettbenign},
and ``ridgeless'' kernel least squares is studied in
\cite{2019arXiv190810292L,2018arXiv181211167R} and \cite{liang2020just}.
While these papers study upper and lower bounds on the risk of interpolating
solutions to the linear and kernel least squares problem, ours are the first to derived using stability arguments. While it might be possible to obtain tighter excess risk bounds through
careful analysis of the minimum norm interpolant, our simple approach helps us
establish a link between stability in  statistical and in  numerical sense.

Finally, we can compare our results with observations made in \cite{DescentCondition}
on the condition number of random kernel matrices. The condition number of the empirical
kernel matrix is known to control the numerical stability of the solution to a kernel least squares
problem. Our results show that the statistical stability is also controlled by the condition number
of the kernel matrix, providing a natural link between numerical and statistical stability.

\begin{figure}
	\centering
	\includegraphics[trim = 110 240 125 255, width=0.5\linewidth, clip]{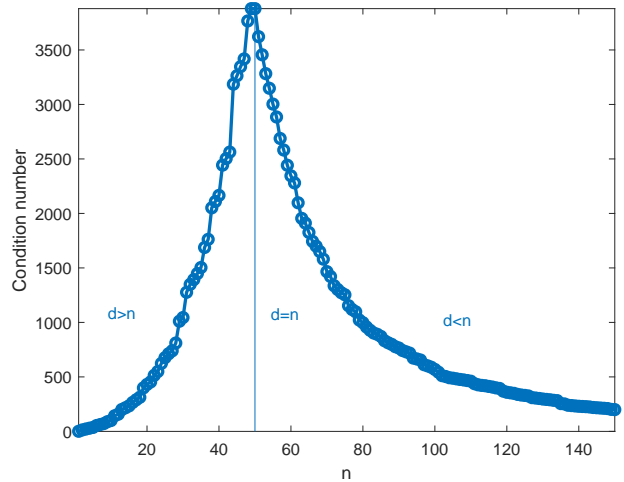}
	\caption{Typical double descent of the condition number (y
		axis) of a radial basis function kernel $K(x,x') = \exp\left(-\frac{||x-x'||^2}{2\sigma^2}\right)$ built from a random data matrix distributed as $\mathcal{N}(0,1)$: as in the linear case, the condition number is
		worse when $n=d$, better
		if $n>d$ (on the  right of $n=d$) and also better if $n<d$ (on the
		left of $n=d$). The parameter $\sigma$ was chosen to
                be 5. From \cite{DescentCondition}}
	\label{CondNumberrbf}
\end{figure}

\section{Conclusions} \label{sec:conclusions}
	\vskip0.1in

        In summary, minimizing a bound on cross validation stability
        minimizes the expected error in both the classical and the
        modern regime of ERM.  In the classical regime ($d<n$), $\cv$
        stability implies generalization and consistency for
        $n\to \infty$.  In the modern regime ($d>n$), as described in this
        paper, $\cv$ stability can account for the double descent
        curve in
        kernel interpolants (\citep{Belkin15849} showed empirical
        evidence for closely related interpolation methods) under appropriate distributional assumptions.
        The main contribution of this paper is characterizing
        stability of interpolating solutions, in particular deriving
        excess risk bounds via a stability argument.  In the process,
        we show that among all the interpolating solutions, the one
        with minimal norm also minimizes a bound on stability. Since
        the excess risk bounds of the minimum norm interpolant depend
        on the pseudoinverse of the kernel matrix, we establish here
        a clear link between {\it numerical and statistical} stability.
        This also holds for solutions computed by gradient descent,
        since gradient descent converges to minimum norm solutions in
        the case of ``linear'' kernel methods.  Our approach is simple
        and combines basic stability results with matrix inequalities.

\newpage

\bibliographystyle{iclr2021_conference}

\bibliography{Boolean}
\normalsize

\newpage

\appendix
\section{Excess Risk, Generalization, and Stability} \label{app:ro_stable}

We use the same notation as introduced in Section \ref{sec:sl_erm_definitions} for the various quantities considered in this section.
That is in the supervised learning setup $V(f,z)$ is the loss incurred by hypothesis $f$ on the sample $z$, and $I[f] = \E_z [V(f,z)]$
is the expected error of hypothesis $f$. Since we are interested in different forms of stability, we will consider learning problems
over the original training set $S=\{z_1, z_2, \ldots, z_n \}$, the leave one out training set $\Si = \{z_1, \ldots, z_{i-1},z_{i+1}, \ldots, z_n \}$, and the replace one training set $(\Si,z) = \{z_1, \ldots, z_{i-1},z_{i+1}, \ldots, z_n, z \}$

\subsection{Replace one and leave one out algorithmic stability}

Similar to the definition of expected $\cv$ stability in equation \eqref{cvst} of the main paper, we say an algorithm
is cross validation {\em replace one} stable (in expectation), denoted as $\cvr$, if
there exists $\beta_{ro}>0$ such that
$$
 \E_{S,z}[V(f_S,z)- V(f_{(\Si, z)}, z)]\le \beta_{ro}.
$$

We can strengthen the above stability definition by introducing the notion of replace one
algorithmic stability (in expectation)  \cite{BE:2001}.
There exists $\alpha_{ro}>$ such that for all $i=1, \dots, n$,
$$
 \E_{S,z}[\nor{f_S-f_{(\Si, z)}}_\infty]\le \alpha_{ro}.
$$
We make two observations:\\
First, if the loss is Lipschitz, that is if  there exists $C_V>0$ such that for all $f,f'\in \hh$
$$
\nor{V(f,z)- V(f', z)}\le C_V\nor{f-f'},
$$
then replace one algorithmic stability implies $\cvr$ stability with $\beta_{ro}= C_V\alpha_{ro}$. Moreover, the same result
holds if the loss is locally Lipschitz and there exists $R>0$, such that  $\nor{f_S}_\infty \le R$ almost surely. In this latter case the Lipschitz constant will depend on $R$.
Later, we illustrate this situation for the square loss.

Second, we have for all $i=1, \dots, n$, $S$ and $z$,
$$
 \E_{S,z}[\nor{f_S-f_{(\Si,z)}}_\infty]\le  \E_{S,z}[\nor{f_S-f_{\Si}}_\infty]+ \E_{S,z}[\nor{f_{(\Si,z)}-f_{\Si}}_\infty].
$$
This observation motivates the notion of leave one out algorithmic stability (in expectation)  \cite{BE:2001}]
$$
 \E_{S,z}[\nor{f_S-f_{\Si}}_\infty]\le \alpha_{loo}.
$$
Clearly,  leave one out algorithmic stability implies replace one algorithmic stability with $\alpha_{ro}=2\alpha_{loo}$ and it implies also $\cvr$ stability with $\beta_{ro}= 2 C_V\alpha_{loo}$.

\subsection{Excess Risk and $\cv$, $\cvr$ Stability}

We recall the statement of Lemma \ref{expriskgen} in section \ref{sec:stability_err_bds} that bounds the excess risk using the $\cv$ stability of a solution.

\begin{lemma}[Excess Risk  \& $\cv$ Stability]\label{app:expriskgen}
For all $i=1, \dots, n$,
\begin{equation}\label{app:cvlooexcess}
\E_S[I[f_{\Si}]-\inf_{f\in \hh} I[f]]\le
\E_S[V(f_{\Si}, z_i)- V(f_S, z_i)].
\end{equation}
\end{lemma}

In this section,  two properties of ERM are useful, namely symmetry, and a form of unbiasedeness.
\paragraph{Symmetry.}  A key property of ERM  is that it is {\em symmetric} with respect
to the data set  $S$, meaning that it does not depend on the order of the data in $S$.


A second property relates the expected ERM with the minimum of expected risk.
\paragraph{ERM Bias.}
The following inequality holds.
\begin{equation}\label{negbias}
\E[[I_S[f_S]] -  \min_{f\in \hh} I[f] \le 0.
\end{equation}
%
To see this, note that
$$
I_S[f_S]\le I_S[f]
$$
for all $f\in\hh$ by definition of ERM, so that
 taking the expectation of both sides
$$
\E_S[I_S[f_S]]\le \E_S[I_S[f]]=I[f]
$$
for all $f\in \hh$. This implies
$$
\E_S[I_S[f_S]]\le\min_{f\in \hh}I[f]
$$
and hence~\eqref{negbias} holds.
\begin{remark}
Note that the same argument gives more generally that
\begin{equation}\label{unbias}
\E[\inf_{f\in \hh}[I_S[f]] -  \inf_{f\in \hh} I[f] \le 0.
\end{equation}
\end{remark}

%

Given the above premise, the proof of Lemma \ref{expriskgen} is simple.
\begin{proof}[of  Lemma \ref{expriskgen}]
Adding and subtracting $\E_S[I_S[f_S]]$ from the expected excess risk we have that
\begin{equation}\label{cvloo1}
\E_S[I[f_{\Si}]- \min_{f\in \hh} I[f]]= \E_S[I[f_{\Si}]-  I_S[f_S]+ I_S[f_S]- \min_{f\in \hh} I[f]],
\end{equation}
and since  $\E_S[I_S[f_S]]- \min_{f\in \hh} I[f]]$ is less or equal than zero, see~\eqref{unbias}, then
\begin{equation}\label{cvloo2}
\E_S[I[f_{\Si}]- \min_{f\in \hh} I[f]]\le \E_S[I[f_{\Si}]-  I_S[f_S]].
\end{equation}
Moreover,   for all $i=1, \dots, n$
$$
 \E_S[I[f_{\Si}]]= \E_S[\E_{z_i}V(f_{\Si}, z_i)] =  \E_S[V(f_{\Si}, z_i)]
$$
and
$$
 \E_S[I_S[f_S]]= \frac 1 n \sum_{i=1}^n \E_S[V(f_S, z_i)]=\E_S[V(f_S, z_i)].
$$
Plugging these last two expressions in~\eqref{cvloo2} and in~\eqref{cvloo1} leads to~\eqref{cvlooexcess}.
\end{proof}

We can prove a similar result relating excess risk with $\cvr$ stability.
\begin{lemma}[Excess Risk \& $\cvr$ Stability]
Given the above definitions, the following inequality holds for all $i=1, \dots, n$,
\begin{equation}
\E_S[I[f_{S}]-\inf_{f\in \hh} I[f]]\le
\E_S[I[f_{S}]-I_S[f_S]] = \E_{S,z}[V(f_S, z)- V(f_{(\Si,z)}, z)].
\end{equation}
\end{lemma}
\begin{proof}
The first inequality is clear from adding and subtracting $I_S[f_S]$ from the expected risk $I[f_S]$ we have that
$$
\E_S[I[f_S]- \min_{f\in \hh} I[f]]= \E_S[I[f_S]-  I_S[f_S]+ I_S[f_S]- \min_{f\in \hh} I[f]], $$
and recalling~\eqref{unbias}. The main step in the proof is showing that for all $i=1,  \dots, n$,
\begin{equation}\label{clever}
\E[I_S[f_S]]= \E[V(f_{(\Si,z)},z)]
\end{equation}
to be compared with the trivial equality, $\E[I_S[f_S]= \E[V(f_{S},z_i)]$.
To prove Equation~\eqref{clever}, we have for all $i=1, \dots, n$,
$$
\E_S[I_S[f_S]]= \E_{S,z}[\frac 1 n \sum_{i=1}^nV(f_S, z_i)]=  \frac 1 n\sum_{i=1}^n  \E_{S,z} [V(f_{(\Si,z)}, z)]= \E_{S,z} [V(f_{(\Si,z)},z)]
$$
where we used the fact that by the symmetry of the algorithm $\E_{S,z} [V(f_{(\Si,z)}, z)]$ is the same for all $i=1, \dots, n$. The proof is concluded noting that $\E_S[I[f_S]]= \E_{S,z}[V(f_S,z)]$.
\end{proof}

\subsection{Discussion on Stability and Generalization}

Below we
discuss some more aspects of stability and its connection to other quantities in statistical
learning theory.

\begin{remark}[$\cv$ stability in expectation and in probability]
In \cite{Mukherjee2006}, $\cv$ stability  is defined in probability, that is
there exists $\beta^P_{CV} >0$,  $0<\delta^P_{CV}\le 1$ such that
 $$
 \PP_S\{ |V(f_{\Si}, z_i)- V(f_S, z_i)| \ge \beta^P_{CV}\}\le \delta^P_{CV}.
 $$
 Note that the absolute value is not needed for ERM since almost positivity
 holds \cite{Mukherjee2006}, that is
 $
 V(f_{\Si}, z_i)- V(f_S, z_i)>0.
 $
Then  $\cv$ stability in probability and in expectation  are clearly related
and indeed equivalent for bounded loss functions. $\cv$  stability in
expectation~\eqref{cvst} is what we study in the following sections.
\end{remark}

\begin{remark}[Connection to  uniform stability and other notions of stability]
  Uniform stability, introduced by \cite{BE:2001}, corresponds in our
  notation to the assumption that there exists $\beta_u>0$ such that
  for all $i=1, \dots, n$,
  $ \sup_{z\in Z} |V(f_{\Si}, z)- V(f_S,z)|\le \beta_u.  $ Clearly
  this is a strong notion implying most other definitions of
  stability.  We note that there are number of different notions of
  stability. We refer the interested reader to \cite{kn2002} ,
  \cite{Mukherjee2006}.
\end{remark}

\begin{remark}[$\cv$ Stability \& Learnability]
  A natural question is to which extent suitable notions of stability
  are not only sufficient but also necessary for controlling the
  excess risk of ERM.  Classically, the latter is characterized in
  terms of a uniform version of the law of large numbers, which itself
  can be characterized in terms of suitable complexity measures
  of the hypothesis class.
  Uniform stability is too strong to characterize consistency
    while $\cv$ stability turns out to provide a suitably weak definition as
    shown in \cite{Mukherjee2006}, see also \cite{kn2002},
    \cite{Mukherjee2006}.  Indeed, a main result in
    \cite{Mukherjee2006} shows that $\cv$ stability is {\it equivalent
    to consistency of ERM}:
\end{remark}

\begin{theorem}\cite{Mukherjee2006} \label{keytp}
For ERM and bounded loss functions, $\cv$
  stability in probability with $\beta^P_{CV}$
  converging to zero for $n \to \infty$ is equivalent to  consistency
 and generalization of ERM.
\end{theorem}

\begin{remark}[$\cv$ stability \& in-sample/out-of-sample error]
Let $(S,z)=  \{z_1,  \dots, z_n , z\},$ ($z$ is a data point drawn according to the same distribution) and the corresponding ERM solution $f_{(S,z)}$,
then~\eqref{cvlooexcess}  can be equivalently written as,
$$
\E_S[I[f_{S}]-\inf_{f\in \cal F} I[f]]\le \E_{S,z}[V(f_S, z)- V(f_{(S,z)}, z)].
$$

Thus $\cv$ stability measures how much the loss changes when we test on a point
that is present in the training set and absent from it. In this view,  it can be seen
as an average measure of the difference between in-sample and out-of-sample error.
\end{remark}

\begin{remark}[$\cv$ stability and generalization]
A common error measure is the (expected) generalization gap
$
\E_S[I[f_S]-  I_S[f_S]].
$
For non-ERM algorithms, $\cv$ stability by itself not sufficient to
control this term, and further conditions are needed
\cite{Mukherjee2006}, since
$$
\E_S[I[f_S]-  I_S[f_S]]= \E_S[I[f_S]-  I_S[f_{\Si}] ] +\E_S[ I_S[f_{\Si}]- I_S[f_S]].
$$
The second term becomes for all $i=1, \dots, n$,
$$
\E_S[ I_S[f_{\Si}]- I_S[f_S]]= \frac 1 n \sum_{i=1}^n \E_S[V(f_{\Si}, z_i) -V(f_{S}, z_i) ] =  \E_S[V(f_{\Si}, z_i) -V(f_{S}, z_i) ]
$$
and hence is controlled by CV stability. The first term is called
expected leave one out error in \cite{Mukherjee2006} and is
controlled in ERM as $n\to \infty$, see  Theorem~\ref{keytp} above.
\end{remark}

\section{$\cv$ Stabililty of Linear Regression} \label{app:linreg_stability}


We have a dataset $S=\{(\x_i, y_i)\}_{i=1}^n$ and we want to find a mapping $\w \in \R^d$,
that minimizes the empirical least squares risk. All interpolating solutions are of the form
$\hat{\w}_S = \y^\top \Xn^\dagger +\vvec^\top (\mathbf{I} - \Xn \Xn^\dagger )$. Similarly, all interpolating
solutions on the leave one out dataset $\Si$ can be written as $\hat{\w}_{\Si} = \y_i^\top (\Xn_i)^\dagger
+ \vvec_i^\top (\mathbf{I} - \Xn_i(\Xn_i)^\dagger)$.
Here $\Xn, \Xn_i \in \R^{d\times n}$ are the data matrices for the original and leave one out datasets respectively.
We note that when $\vvec = \mathbf{0}$ and $\vvec_i = \mathbf{0}$, we obtain the minimum norm
interpolating solutions on the datasets $S$ and $\Si$.

In this section we want to estimate the $\cv$ stability of the minimum norm solution to
the ERM problem in the linear regression case. This is the case outlined in Remark \ref{rem:linreg} of the main paper.
In order to prove Remark \ref{rem:linreg}, we only need to combine Lemma \ref{ls_lip} with the linear regression analogue
of Lemma \ref{lemma:kernel_stability}. We state and prove that result in this section. This result predicts a double descent
curve for the norm of the pseudoinverse as found in practice, see Figure \ref{fig:pinv_norm}.

\begin{lemma} \label{lemma:linear_stability}
Let $\hat{\w}_S, \hat{\w}_{\Si}$ be any interpolating least squares solutions on
the full and leave one out datasets $S, \Si$,
then $\nor{\hat{\w}_S - \hat{\w}_{\Si} } \leq B_{CV}(\Xn^\dagger, \y, \vvec, \vvec_i)$, and
$B_{CV}$ is minimized when $\vvec=\vvec_i =\mathbf{0}$, which corresponds to the minimum
norm interpolating solutions $\w_S^\dagger, \w_{\Si}^\dagger$.

Also,
\begin{equation}\label{eqn:linear_perturbation_norm}
\nor{\w_S^\dagger - \w_{\Si}^\dagger } \leq \nor{\Xn^\dagger}_{op} \nor{\y}
\end{equation}
\end{lemma}

As mentioned before in section \ref{subsec:KLS_min_norm} of the main paper, linear regression can be viewed as
a case of the kernel regression problem where $\hh = \R^d$, and the feature map $\Phi$ is the identity map.
The inner product and norms considered in this case are also the usual Euclidean
inner product and $2$-norm for vectors in $\R^d$.  The notation $\nor{\cdot}$
denotes the Euclidean norm for vectors both in $\R^d$ and $\R^n$.
The usage of the norm should be clear from the context.
Also, $\nor{\mathbf{A}}_{op}$ is the left operator norm
for a matrix $\mathbf{A} \in \R^{n\times d}$, that is $\nor{\mathbf{A}}_{op}
= \sup_{\y \in \R^n, ||\y||=1} ||\y^\top \mathbf{A}||$.

We have $n$ samples in the training set for a linear regression problem,
$\{(\x_i,y_i)\}_{i=1}^n$. We collect all the samples into a single matrix/vector
$\Xn=[ \x_1 \x_2 \ldots \x_n] \in \R^{d\times n}$, and $\y=[y_1 y_2 \ldots y_n]^\top \in \R^n$.
Then any interpolating ERM solution $\w_S$ satisfies the linear equation
\begin{equation}
\w_S^\top \Xn=\y^\top
\end{equation}
Any interpolating solution can be written as:
\begin{equation}
(\hat{\w}_S)^\top= \y^\top \Xn^\dagger + \vvec^\top (\mathbf{I} - \Xn \Xn^\dagger).
\end{equation}

\begin{figure}
	\centering
	\includegraphics[scale=0.6]{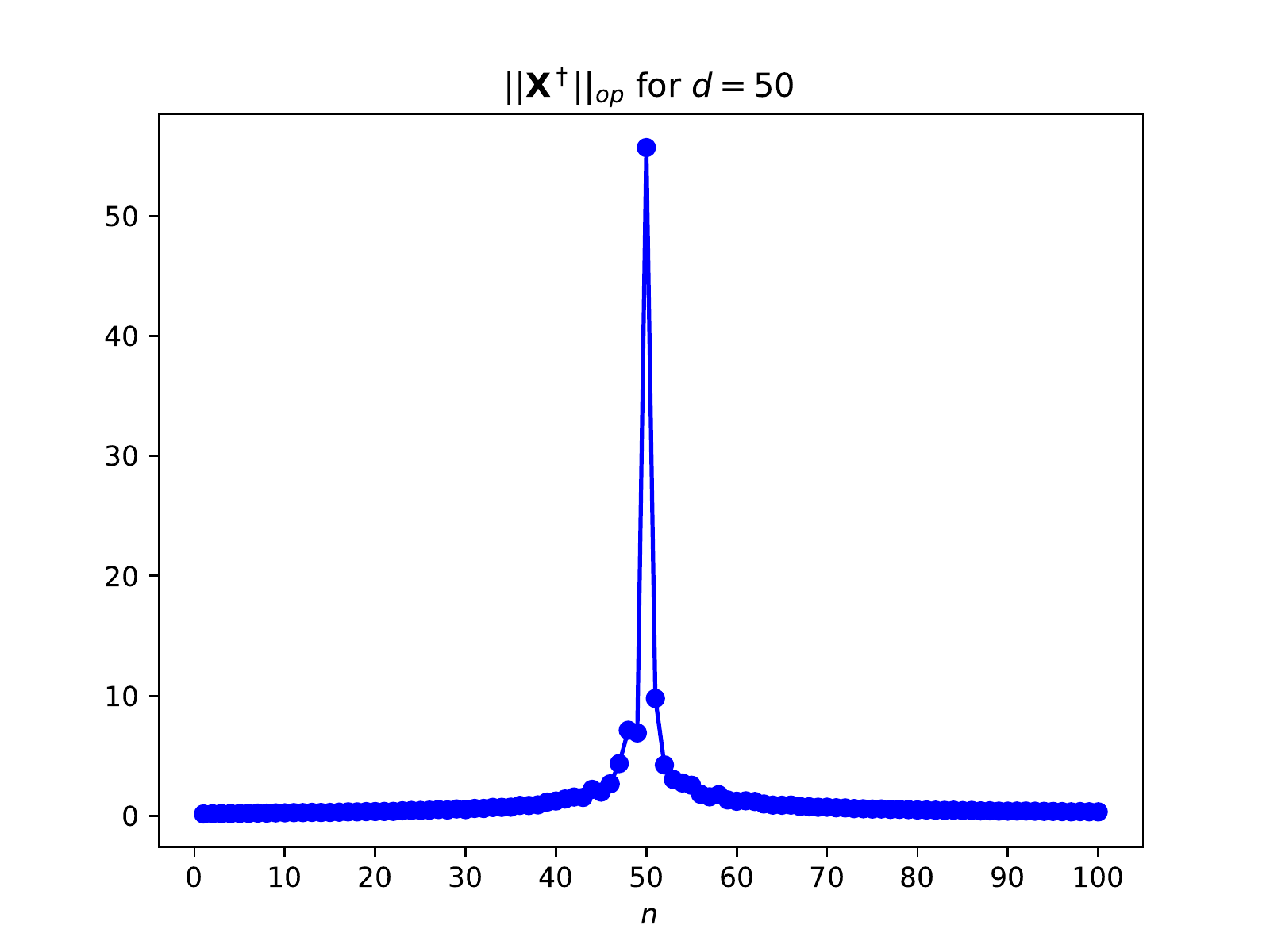}
	\caption{Typical double descent of the pseudoinverse norm (y
          axis) of a random data matrix distributed as $\mathcal{N}(0,1)$: the condition number is
          worse when $n=d$, better
          if $n>d$ (on the  right of $n=d$) and also better if $n<d$ (on the
          left of $n=d$).. From \cite{DescentCondition}}
	\label{fig:pinv_norm}
\end{figure}

\noindent If we consider the leave one out training set $\Si$
we can find the minimum norm ERM solution for $\Xn_i = [ \x_1 \ldots \mathbf{0}
\ldots \x_n ]$ and $\y_i = [y_1 \ldots 0 \ldots y_n]^\top$ as
\begin{equation}
(\hat{\w}_{\Si})^\top = \y_i^\top (\Xn_i)^\dagger + \vvec_i^\top (\mathbf{I} - \Xn_i (\Xn_i)^\dagger) .
\end{equation}
We can write $\Xn_i$ as:
\begin{equation}
\Xn_i= \Xn+\avec \bvec^\top
\end{equation}
where $\avec \in \R^d$ is a column
vector representing the additive change to the
$i^{\textrm{th}}$ column, i.e, $\avec= - \x_i$,
and $\bvec \in \R^{n \times 1}$ is the $i-$th element of the canonical basis in $\R^n$ (all the coefficients are zero but the $i-$th which is one).
Thus $\avec\bvec^\top$ is a $d \times n$ matrix composed of
all zeros apart from the $i^{\textrm{th}}$ column which is equal to $\avec$.

We also have $\y_i = \y - y_i \bvec$. Now per Lemma \ref{ls_lip} we are interested in bounding the quantity
$||\hat{\w}_{\Si} -\hat{\w}_S|| = ||(\hat{\w}_{\Si})^\top - (\hat{\w}_S)^\top ||$. This simplifies to:

\begin{equation} \label{eq:lin_reg_norm_diff}
\begin{split}
    ||\hat{\w}_{\Si} -\hat{\w}_S|| &= || \y_i^\top (\Xn_i)^\dagger - \y^\top \Xn^\dagger + \vvec_i^\top - \vvec^\top + \vvec^\top \Xn \Xn^\dagger - \vvec_i^\top \Xn_i (\Xn_i)^\dagger || \\
    &= ||(\y^\top - y_i \bvec^\top ) (\Xn_i)^\dagger - \y^\top \Xn^\dagger + \vvec_i^\top - \vvec^\top + \vvec^\top \Xn \Xn^\dagger - \vvec_i^\top \Xn_i (\Xn_i)^\dagger|| \\
    &= ||\y^\top ((\Xn_i)^\dagger - \Xn^\dagger) + y_i \bvec^\top (\Xn_i)^\dagger + \vvec_i^\top - \vvec^\top + \vvec^\top \Xn \Xn^\dagger - \vvec_i^\top \Xn_i (\Xn_i)^\dagger|| \\
    &= ||\y^\top ((\Xn_i)^\dagger - \Xn^\dagger) + \vvec_i^\top - \vvec^\top + \vvec^\top \Xn \Xn^\dagger - \vvec_i^\top \Xn_i (\Xn_i)^\dagger|| \\
    &= ||\y^\top ((\Xn_i)^\dagger - \Xn^\dagger) + (\vvec_i^\top - \vvec^\top )(\mathbf{I} - \Xn \Xn^\dagger) - \vvec_i^\top (\Xn \Xn^\dagger - \Xn_i (\Xn_i)^\dagger)||
\end{split}
\end{equation}

In the above equation we make use of the fact that $\bvec^\top (\Xn_i)^\dagger  = \mathbf{0}$.
\noindent We use an old formula \citep{Meyer1973, baksalary2003revisitation}
to compute $(\Xn_i)^\dagger$ from $\Xn^\dagger$. We use the development of pseudo-inverses
of perturbed matrices in \cite{Meyer1973}. We see that
$\avec = - \x_i$ is a vector in the column space of $\Xn$ and
$\bvec$ is in the range space of $\Xn^T$ (provided $\Xn$ has full column rank), with $\beta = 1+ \bvec^\top \Xn^\dagger \avec =
1-\bvec^\top \Xn^\dagger \x_i = 0$. This means
we can use Theorem 6 in \cite{Meyer1973} (equivalent to formula 2.1 in \cite{baksalary2003revisitation})
to obtain the expression for $(\Xn_i)^\dagger$

\begin{equation}
    (\Xn_i)^\dagger = \Xn^\dagger - \kvec\kvec^\dagger \Xn^\dagger - \Xn^\dagger \hvec^\dagger \hvec + (\kvec^\dagger \Xn^\dagger \hvec^\dagger ) \kvec \hvec
\end{equation}

where $\kvec = \Xn^\dagger \avec$, and $\hvec = \bvec^\top \Xn^\dagger$, and $\uvec^\dagger = \frac{\uvec^\top}{||\uvec||^2}$ for any non-zero vector $\uvec$.

\begin{equation}
    \begin{split}
        (\Xn_i)^\dagger - \Xn^\dagger &= (\kvec^\dagger \Xn^\dagger \hvec^\dagger ) \kvec \hvec - \kvec \kvec^\dagger \Xn^\dagger - \Xn^\dagger \hvec^\dagger \hvec \\
        &= (\bvec^\top \Xn^\dagger \hvec^\dagger) \bvec \hvec - \bvec \hvec - \Xn^\dagger \hvec^\dagger \hvec \\
        \implies ||(\Xn_i)^\dagger - \Xn^\dagger||_{op} &= ||\Xn^\dagger \hvec^\dagger \hvec||_{op} \\
        &\leq ||\Xn^\dagger||_{op}
    \end{split}
\end{equation}

\noindent The above set of inequalities follows from the fact
that the operator norm of a rank 1 matrix is given by $||\uvec \vvec^\top||_{op} = ||\uvec||\times ||\vvec||$, and by noticing that $\kvec = - \bvec$.

\noindent Also, from List 2 of \cite{baksalary2003revisitation} we have
that $\Xn_i (\Xn_i)^\dagger = \Xn \Xn^\dagger - \hvec^\dagger \hvec$.

Plugging in these calculations into equation \ref{eq:lin_reg_norm_diff} we get:
\begin{equation}
\begin{split}
    ||\hat{\w}_{\Si} -\hat{\w}_S|| &= ||\y^\top ((\Xn_i)^\dagger - \Xn^\dagger) + (\vvec_i^\top - \vvec^\top )(\mathbf{I} - \Xn \Xn^\dagger) - \vvec_i^\top (\Xn \Xn^\dagger - \Xn_i (\Xn_i)^\dagger)|| \\
    &\leq B_0 + ||\mathbf{I} - \Xn\Xn^\dagger ||_{op} ||\vvec - \vvec_i|| + ||\vvec_i|| \times ||\hvec^\dagger \hvec||_{op} \\
    & \leq  B_0 + 2 ||\vvec - \vvec_i|| + ||\vvec_i||
\end{split}
\end{equation}
We see that the right hand side is minimized when $\vvec = \vvec_i = \mathbf{0}$. We can also compute $B_0 = ||\Xn^\dagger||_{op} ||\y||$, which concludes the proof of Lemma \ref{lemma:linear_stability}.

%

\end{document}